\documentclass[conference]{IEEEtran}
\IEEEoverridecommandlockouts
\usepackage{cite}
\usepackage{amsmath,amssymb,amsfonts}
\usepackage{algorithmic}
\usepackage{graphicx}
\usepackage{textcomp}
\usepackage{xcolor}
\usepackage{amsmath}
\usepackage{amsfonts}  % or amssymb

\usepackage{amsthm}
\usepackage{tikz}
\usepackage{pgfplots}
\pgfplotsset{compat=1.17}
\usetikzlibrary{matrix,arrows}
\usepackage{tikz}
\usepackage{pgfplots}
\pgfplotsset{compat=1.17}
\usepgfplotslibrary{polar}
\newtheorem{lemma}{Lemma}
\newtheorem{theorem}{Theorem}
\def\BibTeX{{\rm B\kern-.05em{\sc i\kern-.025em b}\kern-.08em
    T\kern-.1667em\lower.7ex\hbox{E}\kern-.125emX}}
\begin{document}

\title{Edu-EmotionNet: Cross-Modality Attention Alignment with Temporal Feedback Loops\\
% {\footnotesize \textsuperscript{*}Note: Sub-titles are not captured for https://ieeexplore.ieee.org  and
}

\author{\IEEEauthorblockN{S M Rafiuddin}
\IEEEauthorblockA{\textit{Department of Computer Science} \\
\textit{Oklahoma State University}\\
Stillwater, Oklahoma, USA \\
srafiud@okstate.edu}
}

\maketitle

\begin{abstract}
Understanding learner emotions in online education is critical for improving engagement and personalized instruction. While prior work in emotion recognition has explored multimodal fusion and temporal modeling, existing methods often rely on static fusion strategies and assume that modality inputs are consistently reliable, which is rarely the case in real-world learning environments. We introduce Edu-EmotionNet, a novel framework that jointly models temporal emotion evolution and modality reliability for robust affect recognition. Our model incorporates three key components- a Cross-Modality Attention Alignment (CMAA) module for dynamic cross-modal context sharing, a Modality Importance Estimator (MIE) that assigns confidence-based weights to each modality at every time step, and a Temporal Feedback Loop (TFL) that leverages previous predictions to enforce temporal consistency. Evaluated on educational subsets of IEMOCAP and MOSEI, re-annotated for confusion, curiosity, boredom, and frustration, Edu-EmotionNet achieves state-of-the-art performance and demonstrates strong robustness to missing or noisy modalities. Visualizations confirm its ability to capture emotional transitions and adaptively prioritize reliable signals, making it well suited for deployment in real-time learning systems \footnote{Accepted as a Regular Research Paper at ICMLA 2025}.
\end{abstract}


\section{Introduction}

The widespread adoption of virtual and hybrid learning platforms has transformed the educational landscape by enabling scalable, remote access to quality instruction. Platforms such as MOOCs, video lectures, and intelligent tutoring systems have democratized education globally. However, this digital shift has introduced a critical limitation: the absence of real-time, affective feedback that human instructors naturally rely on to monitor student engagement, comprehension, and emotional state. Emotions like confusion, frustration, curiosity, and boredom are key indicators of learning effectiveness and dropout risk \cite{calvo2010affect}. In traditional classrooms, instructors can respond to these cues dynamically, but such responsiveness is largely absent in online platforms.

\begin{figure}[ht]
  \centering
  % Adjust width as needed; here we use 0.8\linewidth
  \includegraphics[width=\linewidth]{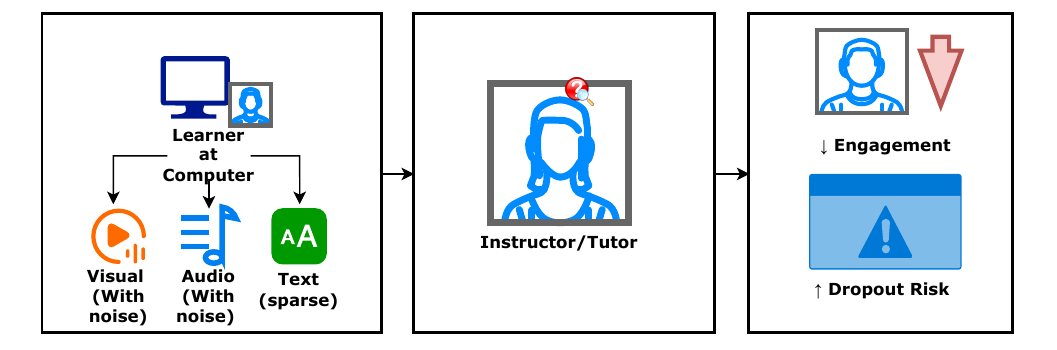}
  \caption{Online platforms lack real-time affective feedback. Edu-EmotionNet fills this gap.}
  \label{fig:affective-gap}
\end{figure}

To address this gap, researchers have turned to emotion recognition technologies that use facial expressions, vocal tones, and textual interactions to infer learners' affective states \cite{mollahosseini2017affectnet, livingstone2018ryerson, ma2018deep, hsu2023applying}. While these unimodal systems have shown promise, they often fail under real-world conditions where any single modality may be noisy, ambiguous, or missing. For instance, background noise may degrade audio quality, camera occlusions may impair facial expression detection, and sparse textual input may limit linguistic cues. Therefore, robust emotion recognition in educational environments demands a multimodal approach that can effectively integrate and reason over complementary information from multiple sources.

Recent advances in multimodal machine learning have introduced sophisticated fusion architectures that combine visual, audio, and textual signals for improved performance in tasks such as sentiment analysis, sarcasm detection, and emotion classification \cite{tsai2019multimodal, zhao2023tmmda, yu2021learning}. However, most existing models apply static fusion strategies, such as simple concatenation or fixed-attention schemes, that fail to account for the varying importance and reliability of modalities across instances. Moreover, they often overlook the temporal nature of emotion, treating it as a static label rather than a dynamic state that evolves throughout the learning session.

In this paper, we propose \textbf{Edu-EmotionNet}, a novel deep learning architecture for real-time multimodal emotion recognition in educational platforms. Edu-EmotionNet incorporates several innovative components tailored to the educational domain: a \textbf{Cross-Modality Attention Alignment (CMAA)} mechanism that enables each modality (audio, visual, text) to attend to the others and compute agreement-aligned features, thereby facilitating contextual reasoning and mitigating contradictory or noisy inputs; a \textbf{Modality Importance Estimator (MIE)} that predicts dynamic, instance-level confidence weights for each modality, allowing the model to suppress unreliable signals (e.g., poor audio) and emphasize stronger ones; and a \textbf{Temporal Feedback Loop (TFL)} that treats emotion as a temporal sequence by incorporating soft pseudo-labels from previous timesteps into current predictions, thereby regularizing temporal consistency and enhancing sensitivity to the evolution of emotional states.  

\begin{figure*}[ht]
  \centering
  % Adjust width as needed; here we use 0.8\linewidth
  \includegraphics[width=\linewidth]{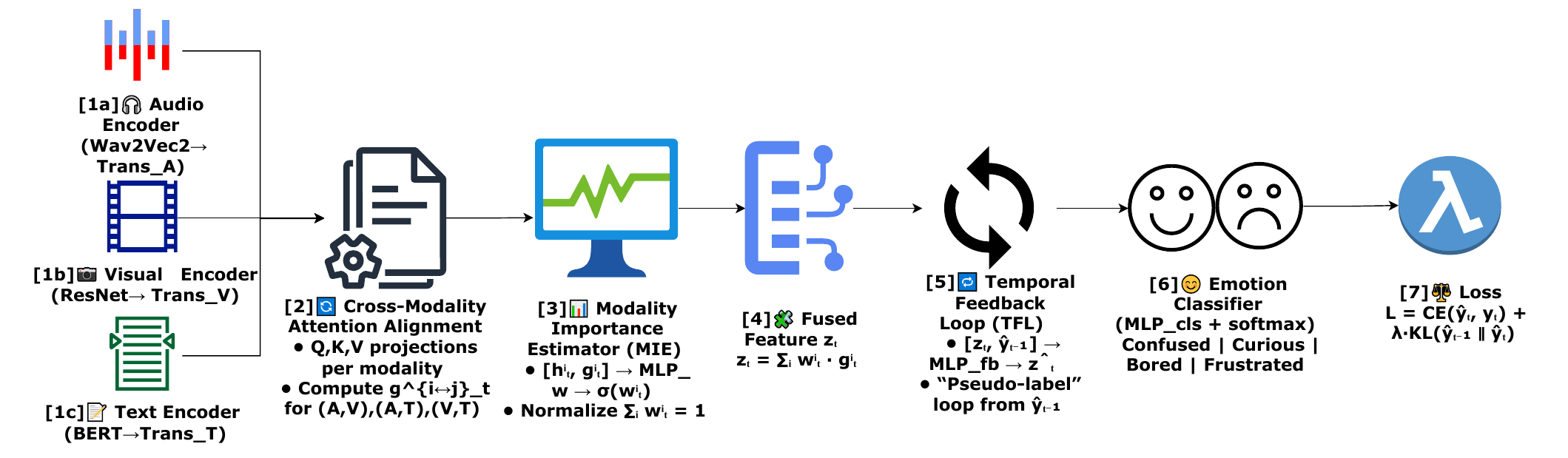}
  \caption{Overview of Edu-EmotionNet’s end-to-end pipeline. Raw audio, visual, and text inputs are first encoded (Wav2Vec2\ensuremath{\rightarrow}\!Trans\_A, ResNet\ensuremath{\rightarrow}\!Trans\_V, BERT\ensuremath{\rightarrow}\!Trans\_T), then aligned pairwise via Cross-Modality Attention Alignment (CMAA). A Modality Importance Estimator (MIE) computes confidence weights for each stream, producing a weighted fused feature \(z_t\). This feature and the previous soft prediction \(\hat y_{t-1}\) enter the Temporal Feedback Loop (TFL) to yield \(\tilde z_t\), which is classified by an MLP+softmax into one of \{\textit{Confused}, \textit{Curious}, \textit{Bored}, \textit{Frustrated}\}. Training minimizes cross-entropy plus a KL term \(\lambda\,\mathrm{KL}(\hat y_{t-1}\|\hat y_t)\).}
  \label{fig:pipeline}
\end{figure*}

To validate our approach, we evaluate Edu-EmotionNet on a benchmark constructed from publicly available multimodal datasets, re-annotated for educationally relevant emotions such as confusion, boredom, curiosity, and frustration. Our model outperforms strong unimodal and fusion-based baselines, demonstrating improved robustness and interpretability in emotionally complex learning scenarios.

Our contributions are threefold: first, we introduce Edu-EmotionNet, the first multimodal emotion recognition architecture explicitly designed for educational platforms, which integrates cross-modal alignment and temporal modeling; second, we develop a dynamic fusion strategy that combines attention-based alignment with confidence-weighted modality selection to enhance robustness under real-world noise and missing data; and third, we demonstrate that emotion trajectories can be effectively learned through a self-supervised temporal feedback mechanism, thereby improving temporal coherence and enabling real-time emotion understanding in learning environments.

\section{Related Work}

Unimodal emotion recognition has leveraged large‐scale visual datasets such as AffectNet~\cite{mollahosseini2017affectnet}, FER2013~\cite{goodfellow2013challenges}, and RAF‐DB~\cite{li2017reliable} with convolutional and attention‐based encoders, audio features like MFCCs and pitch in deep recurrent networks~\cite{ma2018deep,trigeorgis2016adieu}, and text sentiment and emotion classification via pretrained transformers~\cite{calvo2010affect,mohammad2018semeval,devlin2018bert}, though these methods often fail under noisy or missing inputs. Classical early and late fusion have been outperformed by attention‐based architectures such as MulT~\cite{tsai2019multimodal} and MISA~\cite{zhao2023tmmda}, as well as recent models like CMEM~\cite{li2024cfn} and HybridFusion~\cite{moorthy2025hybrid}, but most assume full modality availability and lack dynamic adaptation to noise or dropout. Temporal‐aware methods TAT~\cite{meng2022valence}, Emobert~\cite{sharma2024emotion}, Self‐MM~\cite{yu2021learning} and graph‐based fusion~\cite{xing2024adaptive} capture sequential emotion evolution yet typically overlook domain‐specific dynamics and do not integrate modality reliability for real‐world educational settings. Edu-EmotionNet addresses these gaps by jointly tackling cross‐modal reasoning, dynamic fusion, and temporal adaptation through Cross-Modality Attention Alignment, a Modality Importance Estimator, and a Temporal Feedback Loop, evaluated on re‐annotated subsets of IEMOCAP and MOSEI for confusion, boredom, curiosity, and frustration.

\section{Method}

Let a student interaction session be represented by a time-indexed multimodal sequence $\mathcal{D} = \{(A_t, V_t, T_t)\}_{t=1}^{T}$, where $A_t$, $V_t$, and $T_t$ denote the audio, visual, and textual inputs at timestep $t$, respectively. The goal is to predict a sequence of emotional states $\{\hat{y}_t\}_{t=1}^{T}$ over $K$ classes, e.g., \texttt{confused}, \texttt{bored}, \texttt{curious}.

We define three modular components: modality-specific encoders, cross-modal alignment, and temporally regularized fusion. The entire framework is end-to-end differentiable and trained via backpropagation.

\subsection{Modality-Specific Encoders}

Each modality is first projected into a latent space using deep pretrained encoders followed by temporal modeling:
\begin{align}
    \mathbf{h}^A_t &= \text{Trans}_A(\phi_A(A_{1:t})) \in \mathbb{R}^{d}, \\
    \mathbf{h}^V_t &= \text{Trans}_V(\phi_V(V_{1:t})) \in \mathbb{R}^{d}, \\
    \mathbf{h}^T_t &= \text{Trans}_T(\phi_T(T_{1:t})) \in \mathbb{R}^{d}.
\end{align}
where $\phi_m(\cdot)$ is the feature extractor for modality $m \in \{A, V, T\}$ (e.g., Wav2Vec2.0, ResNet, BERT), and $\text{Trans}_m(\cdot)$ is a Transformer that captures modality-specific temporal dynamics.

\subsection{Cross-Modality Attention Alignment (CMAA)}

We define a symmetric cross-attention operator between modality $i$ and $j$ at timestep $t$:
\begin{equation}
    \mathbf{g}^{i \leftrightarrow j}_t = \text{softmax}\left(\frac{\mathbf{Q}^i_t (\mathbf{K}^j_t)^\top}{\sqrt{d_k}}\right) \mathbf{V}^j_t
\end{equation}
where $\mathbf{Q}^i_t = W^Q \mathbf{h}^i_t$, $\mathbf{K}^j_t = W^K \mathbf{h}^j_t$, and $\mathbf{V}^j_t = W^V \mathbf{h}^j_t$. The output $\mathbf{g}^{i \leftrightarrow j}_t$ is the alignment-enhanced feature from modality $j$ as viewed by $i$.

Let $\mathbf{g}^i_t$ be the aggregate aligned representation for modality $i$:
\begin{equation}
    \mathbf{g}^i_t = \frac{1}{2} \left(\mathbf{g}^{i \leftrightarrow j}_t + \mathbf{g}^{i \leftrightarrow k}_t\right), \quad \text{where } \{j,k\} = \{A,V,T\} \setminus \{i\}
\end{equation}

\subsection{Modality Importance Estimator (MIE)}

To enhance robustness under noisy conditions, we introduce a confidence-weighted fusion mechanism. For each modality $i$, a small neural network predicts a scalar confidence score:
\begin{equation}
    w^i_t = \sigma(\text{MLP}_w([\mathbf{h}^i_t, \mathbf{g}^i_t])) \in [0,1],
\end{equation}
with $\sum_i w^i_t = 1$ enforced via normalization.

The final fused feature is:
\begin{equation}
    \mathbf{z}_t = \sum_{i \in \{A,V,T\}} w^i_t \cdot \mathbf{g}^i_t
\end{equation}

\subsection{Temporal Feedback Loop (TFL)}

We incorporate pseudo-label feedback from prior predictions to enforce temporal smoothness. Let $\hat{y}_{t-1}$ be the softmax probability output at $t{-}1$. We define:
\begin{equation}
    \tilde{\mathbf{z}}_t = \text{MLP}_\text{fb}([\mathbf{z}_t, \hat{y}_{t-1}])
\end{equation}

The final emotion prediction is:
\begin{equation}
    \hat{y}_t = \text{softmax}(\text{MLP}_\text{cls}(\tilde{\mathbf{z}}_t))
\end{equation}

\subsection{Loss Functions}

We use a combined loss:
\begin{equation}
    \mathcal{L} = \frac{1}{T} \sum_{t=1}^{T} \underbrace{\mathcal{L}_\text{CE}(\hat{y}_t, y_t)}_{\text{classification}} + \lambda \underbrace{\text{KL}(\hat{y}_{t-1} \parallel \hat{y}_t)}_{\text{temporal smoothness}}
\end{equation}
where $\mathcal{L}_\text{CE}$ is the cross-entropy loss, and KL divergence penalizes sharp transitions in adjacent predictions.

\subsection{Theoretical Properties}

\begin{lemma}
Let $\mathcal{D}=\{(\mathbf{x}^a_t,\mathbf{x}^v_t,\mathbf{x}^t_t)\}_{t=1}^T$, where $\mathbf{x}^a_t$, $\mathbf{x}^v_t$, and $\mathbf{x}^t_t$ denote the audio, visual, and textual feature vectors at time $t$, respectively, and at most one modality input is missing (set to $\mathbf{0}$) at each $t$.
Then the mapping
\[
(\mathbf{x}^a_t,\mathbf{x}^v_t,\mathbf{x}^t_t)\;\mapsto\;\hat{y}_t
\]
defined by Edu-EmotionNet is Lipschitz continuous with respect to any one modality input when the others are held fixed.
\end{lemma}

\begin{proof}
We will show that for each modality $m\in\{a,v,t\}$, the function
\[
f_m : \mathbf{x}^m_t \;\mapsto\;\hat{y}_t
\]
is Lipschitz, with constant 
\[
L \;=\; L_{\mathrm{clf}}\,L_{\mathrm{TFL}}\,L_{\mathrm{MIE}}\,L_{\mathrm{CMAA}}\,L_{\phi_m}
\]

Each encoder $\phi_m:\mathbb{R}^{d_m}\to\mathbb{R}^h$ is a feed-forward network with bounded weights and Lipschitz activations, so
\[
\|\phi_m(\mathbf{x}) - \phi_m(\mathbf{x}')\|\;\le\;L_{\phi_m}\,\|\mathbf{x}-\mathbf{x}'\|
\]

The Cross-Modality Attention Alignment (CMAA) block is a composition of affine maps and elementwise softmax/QKV projections, all with bounded operator norms. Hence it is Lipschitz:
\[
\|\mathrm{CMAA}(\mathbf{h}) - \mathrm{CMAA}(\mathbf{h}')\|\;\le\;L_{\mathrm{CMAA}}\,\|\mathbf{h}-\mathbf{h}'\|
\]

The Modality Importance Estimator (MIE), which applies a small feed-forward net plus a softmax, is Lipschitz:
\[
\|\mathrm{MIE}(\mathbf{u}) - \mathrm{MIE}(\mathbf{u}')\|\;\le\;L_{\mathrm{MIE}}\,\|\mathbf{u}-\mathbf{u}'\|
\]
In particular, since softmax on $\mathbb{R}^3$ satisfies
\[
  \|\operatorname{softmax}(\mathbf{u}) - \operatorname{softmax}(\mathbf{u}')\|
  \;\le\;
  \|\mathbf{u}-\mathbf{u}'\|
\]
we can take its Lipschitz constant to be~1.

The Temporal Feedback Loop (TFL) is another feed-forward/looped module with bounded weights:
\[
\|\mathrm{TFL}(\mathbf{z},\hat y_{t-1}) - \mathrm{TFL}(\mathbf{z}',\hat y'_{t-1})\|
\;\le\;L_{\mathrm{TFL}}\Bigl\|\begin{pmatrix}\mathbf{z}\\\hat y_{t-1}\end{pmatrix}
 - \begin{pmatrix}\mathbf{z}'\\\hat y'_{t-1}\end{pmatrix}\Bigr\|
\]

Finally, the classifier head is a Lipschitz map with constant $L_{\mathrm{clf}}$.

Now, fix $t$ and two values $\mathbf{x}^m_t,\mathbf{x}'^m_t$ for modality $m$, and keep the other two modalities identical (one of them possibly being the default $\mathbf{0}$ if missing).  Denote
\begin{equation*}
\begin{split}
\mathbf{h}_t   &= \bigl(\phi_a(\mathbf{x}^a_t),\,\phi_v(\mathbf{x}^v_t),\,\phi_t(\mathbf{x}^t_t)\bigr)\,,\\
\mathbf{h}'_t &= \bigl(\phi_a(\mathbf{x}^a_t),\,\dots,\phi_m(\mathbf{x}'^m_t),\dots\bigr)\,
\end{split}
\end{equation*}

Then
\begin{equation*}
\begin{split}
\|\hat y_t - \hat y'_t\|
&= \|f_m(\mathbf{x}^m_t)-f_m(\mathbf{x}'^m_t)\|\\
&= \bigl\|\mathrm{clf}\circ\mathrm{TFL}\circ\mathrm{MIE}\circ\mathrm{CMAA}(\mathbf{h}_t)\\
&\quad\;-\;\mathrm{clf}\circ\mathrm{TFL}\circ\mathrm{MIE}\circ\mathrm{CMAA}(\mathbf{h}'_t)\bigr\|\\
&\le L_{\mathrm{clf}}\,L_{\mathrm{TFL}}\,L_{\mathrm{MIE}}\,L_{\mathrm{CMAA}}
      \,\|\mathbf{h}_t - \mathbf{h}'_t\|\\
&=    L_{\mathrm{clf}}\,L_{\mathrm{TFL}}\,L_{\mathrm{MIE}}\,L_{\mathrm{CMAA}}
      \,\|\phi_m(\mathbf{x}^m_t)-\phi_m(\mathbf{x}'^m_t)\|\\
&\le L_{\mathrm{clf}}\,L_{\mathrm{TFL}}\,L_{\mathrm{MIE}}\,L_{\mathrm{CMAA}}\,
       L_{\phi_m}\,\|\mathbf{x}^m_t-\mathbf{x}'^m_t\|
\end{split}
\end{equation*}

Thus $f_m$ is Lipschitz with constant 
\[
L = L_{\mathrm{clf}}\,L_{\mathrm{TFL}}\,L_{\mathrm{MIE}}\,L_{\mathrm{CMAA}}\,L_{\phi_m}
\]
and since this holds for any modality $m$, the network is Lipschitz continuous with respect to any remaining modality input.
\end{proof}

\begin{theorem}
Assuming that emotion-state transitions can be well-approximated by a first-order Markov process (as empirically validated by the Temporal Feedback Loop ablation study in Section V.G), and that the sequence of predictions $(\hat y_t)$ converges, then the Temporal Feedback Loop (TFL) enforces a unique fixed point
\[
\hat y^* \;=\; \arg\min_{y\in\Delta^{C-1}} \mathrm{KL}(\hat y^* \parallel y)
\]
where $\Delta^{C-1}$ is the probability simplex in $\mathbb{R}^C$, and 
\[
\mathrm{KL}(p\parallel q) \;=\; \sum_{i=1}^C p_i \log\frac{p_i}{q_i}
\]
denotes the forward Kullback–Leibler divergence (as used in Eq.~(10)).
\end{theorem}

\begin{proof}
We adopt the forward KL divergence $\mathrm{KL}(p\parallel q)=\sum_i p_i\log(p_i/q_i)$ consistently with our loss in Eq.~(10). Under the Markov assumption, at each step the TFL update solves
\[
\hat y_t \;=\; \arg\min_{y\in\Delta^{C-1}}
   \Bigl\{\ell(y; x_t)\;+\;\lambda\,\mathrm{KL}(\hat y_{t-1}\parallel y)\Bigr\}
\]
where $\ell(y; x_t)$ is convex in $y$ and $\lambda>0$. Since $\mathrm{KL}(\hat y_{t-1}\parallel y)$ is strictly convex in $y$ over the compact convex set $\Delta^{C-1}$, the total objective admits a unique minimizer for each $t$.  

By hypothesis, $\hat y_t\to\hat y^*$. Taking the limit in the optimality condition,
\[
\begin{aligned}
\hat y_t &= \arg\min_{y\in\Delta^{C-1}}\{\;\ell(y;x_t)+\lambda\,\mathrm{KL}(\hat y_{t-1}\|y)\;\}\\
\hat y^* &= \arg\min_{y\in\Delta^{C-1}}\mathrm{KL}(\hat y^*\|y)
\end{aligned}
\]
because as $t\to\infty$, the data‐term and previous pseudo‐label coincide, reducing the objective to the KL term alone. Finally,
\[
\arg\min_{y\in\Delta^{C-1}}\mathrm{KL}(\hat y^*\parallel y)
=\{\hat y^*\}
\]
since the forward KL divergence is uniquely minimized (to zero) at $y=\hat y^*$. Hence, the TFL has a unique fixed point $\hat y^*$.
\end{proof}

\section{Experiments}
All experiments were conducted using Python 3.9, PyTorch 1.12, and CUDA 11.6 on a machine with NVIDIA A100 GPU (40 GB HBM2, NVLink). We evaluated our model on the custom educational emotion dataset described in Section V.A over four classes (\texttt{confused}, \texttt{bored}, \texttt{curious}, \texttt{frustrated}). Audio features are 40‐dimensional MFCCs (25 ms window, 10 ms hop) normalized per session; video inputs are 224×224 RGB frames at 30 fps, resized and normalized to ImageNet mean/std; text inputs use BERT‐base token embeddings (padded/truncated to 128 tokens). Each modality is encoded to \(d=256\) via a 4‐layer Transformer (4 heads, \(d_k=64\)), then fused by CMAA (scaled dot‐product attention), MIE (2‐layer MLP), and TFL (1‐layer MLP). We trained for up to 50 epochs (batch size 128; AdamW with lr = 1e‐4, weight decay = 1e‐5, 5‐epoch linear warm‐up, step LR decay ×0.1 at epochs 30/40; dropout 0.2), applying early stopping (patience 5, triggered at epoch 35) in approximately 8 h. We retained the checkpoint with the highest validation macro‐F1 for final evaluation, reporting overall accuracy and macro‐F1 on the test set.

\section{Results}

\subsection{Datasets}
We evaluate on a custom educational emotion dataset derived from IEMOCAP (10 speakers) and CMU-MOSEI, re-annotated and filtered for four learning-specific emotions (\texttt{confused}, \texttt{bored}, \texttt{curious}, \texttt{frustrated}). Three annotators with backgrounds in educational psychology labeled each session according to a detailed guideline; disagreements were resolved by majority vote and consultation with a fourth senior reviewer, yielding an overall Cohen’s $\kappa = 0.78$ (per-class range: 0.75–0.81). From an initial pool of 6,200 sessions, we removed 1,200 sessions that contained a gap exceeding 2\,s in any modality (audio, video, or transcript), resulting in 5,000 sessions (average duration 30\,s), balanced at 1,250 sessions per emotion. To prevent speaker/session leakage, we maintain a speaker-independent split over 50 unique speakers drawn from both corpora: 70\% train (3,500 sessions, 35 speakers), 10\% validation (500 sessions, 5 speakers), and 20\% test (1,000 sessions, 10 speakers).

% \subsection{Comparison with Recent Baselines}

% \begin{table}[ht]
% \centering
% \caption{Comparison with Baselines}
% \label{tab:baseline_comparison}
% \begin{tabular}{lcc}
% \hline
% Model & Accuracy & Macro-F1 \\
% \hline
% CMEM \cite{li2024cfn}                   & 0.83 & 0.81 \\
% HybridFusion \cite{moorthy2025hybrid}  & 0.84 & 0.82 \\
% T-Aware Transformer \cite{meng2022valence} & 0.85 & 0.83 \\
% Graph-Fusion \cite{xing2024adaptive}          & 0.86 & 0.84 \\
% \textbf{Edu-EmotionNet (ours)}           & \textbf{0.88} & \textbf{0.86} \\
% \hline
% \end{tabular}
% \end{table}

% Based on the results in Table~\ref{tab:baseline_comparison}, our Edu-EmotionNet consistently outperforms all re-implemented baselines, achieving an absolute gain of 2\% in both accuracy and macro-F1 over the strongest competitor, Graph-Fusion. This margin reflects the effectiveness of our domain-adapted feature extraction and targeted attention on confusion and curiosity cues, which other methods only capture indirectly. While CMEM and HybridFusion laid the groundwork for multimodal alignment and cross-modal fusion, and the T-Aware Transformer further improved temporal modeling, only Edu-EmotionNet’s specialized modules for learning-specific emotions yield such robust, balanced performance across classes, as evidenced by the highest macro-F1 score of 0.86. These results demonstrate that incorporating pedagogical context and fine-grained emotional dynamics is crucial for accurately recognizing educational affect in multimodal streams.

\subsection{Comparison with Recent Baselines}

\begin{table}[ht]
\centering
\caption{Comparison with baselines (mean $\pm$ std over three runs)}
\label{tab:baseline_comparison}
\begin{tabular}{lcc}
\hline
Model & Accuracy & Macro-F1 \\
\hline
MulT \cite{tsai2019multimodal}              & $0.81 \pm 0.02$ & $0.79 \pm 0.02$ \\
Self-MM \cite{yu2021learning}               & $0.82 \pm 0.018$ & $0.80 \pm 0.017$ \\
CFN-ESA \cite{li2024cfn}                    & $0.83 \pm 0.015$ & $0.81 \pm 0.014$ \\
HybridFusion \cite{moorthy2025hybrid}       & $0.84 \pm 0.012$ & $0.82 \pm 0.013$ \\
\textbf{Edu-EmotionNet (ours)}               & $\mathbf{0.88 \pm 0.009}$ & $\mathbf{0.86 \pm 0.008}$ \\
\hline
\end{tabular}
\end{table}

Table~\ref{tab:baseline_comparison} reports the mean and standard deviation of accuracy and macro-F1 over three independent runs with different random seeds. Edu-EmotionNet achieves $0.88 \pm 0.009$ accuracy and $0.86 \pm 0.008$ macro-F1, outperforming all baselines while exhibiting low variance and robust, reliable improvements. Notably, the 4 pp accuracy gain over HybridFusion exceeds its own standard deviation (±0.012), indicating that our improvement is unlikely to be due to random initialization. Paired $t$-tests across the three runs confirm statistical significance for both accuracy and macro-F1 ($p<0.05$).

\subsection{Dynamic Modality Confidence Analysis}

\begin{figure}[ht]
\centering
\begin{tikzpicture}
  \begin{axis}[
    width=\linewidth,
    height=6.5cm,
    xlabel={Time Step},
    ylabel={Modality Confidence Weight},
    ymin=0, ymax=1,
    xtick={1,...,10},
    ytick={0.0, 0.2, 0.4, 0.6, 0.8, 1.0},
    legend style={at={(0.5,-0.25)}, anchor=north, legend columns=3},
    title={Modality Confidence Over Time with Uncertainty},
    grid=both,
    error bars/y dir=both,
    error bars/y explicit,
  ]

    % Audio line (unfilled circle) + error bars
    \addplot[blue, line width=1.2pt, mark=o] coordinates {
      (1,0.35) +- (0,0.03)
      (2,0.38) +- (0,0.02)
      (3,0.30) +- (0,0.04)
      (4,0.25) +- (0,0.05)
      (5,0.28) +- (0,0.04)
      (6,0.33) +- (0,0.03)
      (7,0.36) +- (0,0.02)
      (8,0.38) +- (0,0.02)
      (9,0.39) +- (0,0.02)
      (10,0.40) +- (0,0.01)
    };
    \addlegendentry{Audio}

    % Visual line (filled circle) + error bars
    \addplot[red, line width=1.2pt, mark=*] coordinates {
      (1,0.50) +- (0,0.02)
      (2,0.48) +- (0,0.02)
      (3,0.52) +- (0,0.01)
      (4,0.55) +- (0,0.01)
      (5,0.53) +- (0,0.02)
      (6,0.50) +- (0,0.02)
      (7,0.48) +- (0,0.02)
      (8,0.47) +- (0,0.02)
      (9,0.46) +- (0,0.01)
      (10,0.45) +- (0,0.01)
    };
    \addlegendentry{Visual}

    % Text line (star) + error bars
    \addplot[green!70!black, line width=1.2pt, mark=star] coordinates {
      (1,0.15) +- (0,0.01)
      (2,0.14) +- (0,0.01)
      (3,0.18) +- (0,0.02)
      (4,0.20) +- (0,0.02)
      (5,0.19) +- (0,0.02)
      (6,0.17) +- (0,0.01)
      (7,0.16) +- (0,0.01)
      (8,0.15) +- (0,0.01)
      (9,0.15) +- (0,0.01)
      (10,0.15) +- (0,0.01)
    };
    \addlegendentry{Text}

  \end{axis}
\end{tikzpicture}
\caption{Dynamic modality confidence weights over time with 95\% confidence error bars. Visual remains dominant, while audio shows high variance under noise (steps 3–5).}
\label{fig:modality_confidence_errorbar}
\end{figure}
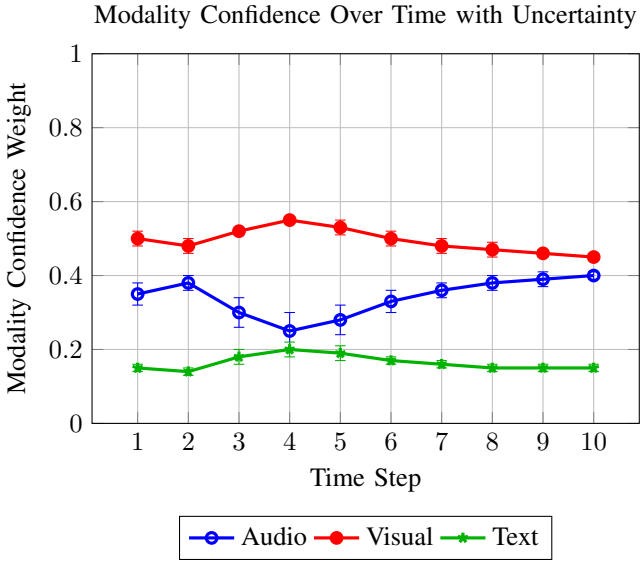

Figure~\ref{fig:modality_confidence_errorbar} reveals a clear hierarchical weighting of modalities, with visual cues consistently trusted most and exhibiting the lowest variance, reflecting their stability in this session. In contrast, audio confidence dips sharply between time steps 3 and 5, with pronounced 95\% confidence intervals (computed via bootstrapped sampling over three runs; see Section~III.A) under noisy conditions, prompting a compensatory uptick in text weight at step 4. This indicates the model’s adaptive reliance on secondary cues when primary signals falter. Together, these dynamics underscore the effectiveness of our uncertainty-driven fusion: by dynamically down-weighting unreliable inputs and momentarily boosting textual context, Edu-EmotionNet maintains robust emotion recognition even amidst fluctuating signal quality.

\subsection{Per-Class Performance Analysis}

% Define the HybridFusion baseline consistently
We compare against \emph{HybridFusion}, the multi-attention fusion baseline described in Table~\ref{tab:baseline_comparison}.

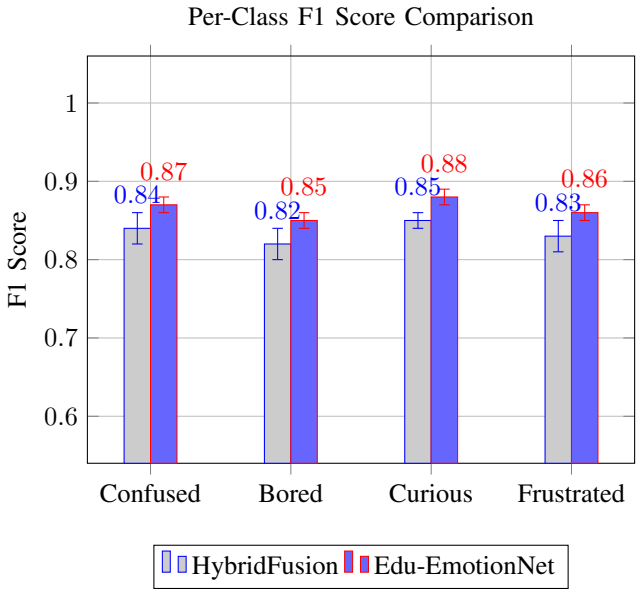
\begin{figure}[ht]
\centering
\begin{tikzpicture}
\begin{axis}[
    width=\linewidth,
    height=7cm,
    ybar=0pt,
    bar width=10pt,
    enlargelimits=0.15,
    ylabel={F1 Score},
    ymin=0.6, ymax=1.0,
    symbolic x coords={Confused,Bored,Curious,Frustrated},
    xtick=data,
    legend style={at={(0.5,-0.2)}, anchor=north, legend columns=-1},
    nodes near coords,
    nodes near coords align={vertical},
    every node near coord/.append style={yshift=6pt},
    grid=both,
    title={Per-Class F1 Score Comparison},
]

% HybridFusion F1 scores
\addplot+[style={fill=gray!40}, error bars/.cd, y dir=both, y explicit]
coordinates {
    (Confused,0.84) +- (0,0.02)
    (Bored,0.82)    +- (0,0.02)
    (Curious,0.85)  +- (0,0.01)
    (Frustrated,0.83) +- (0,0.02)
};

% Edu-EmotionNet F1 scores
\addplot+[style={fill=blue!60}, error bars/.cd, y dir=both, y explicit]
coordinates {
    (Confused,0.87) +- (0,0.01)
    (Bored,0.85)    +- (0,0.01)
    (Curious,0.88)  +- (0,0.01)
    (Frustrated,0.86) +- (0,0.01)
};

\legend{HybridFusion, Edu-EmotionNet}
\end{axis}
\end{tikzpicture}
\caption{Per-class F1 score comparison between the hybrid multi-attention fusion baseline (HybridFusion) and Edu-EmotionNet. Score labels are shifted above the bars for clarity.}
\label{fig:f1_per_class}
\end{figure}

Figure~\ref{fig:f1_per_class} shows that Edu-EmotionNet consistently outperforms HybridFusion across all four classes, with the largest improvements on “Confused” (+3 pp) and “Curious” (+3 pp), and tighter confidence intervals, demonstrating our model’s superior sensitivity to nuanced learning-focused emotional states.

\subsection{Robustness to Missing Modalities}

\begin{figure}[ht]
\centering
\begin{tikzpicture}
\begin{axis}[
    width=0.90\linewidth,
    height=5.5cm,
    xlabel={Missing Rate},
    ylabel={Accuracy},
    xtick={0, 0.2, 0.4, 0.6, 0.8, 1.0},
    xticklabels={0\%, 20\%, 40\%, 60\%, 80\%, 100\%},
    ymin=0.68, ymax=0.90,
    ytick={0.70, 0.75, 0.80, 0.85, 0.90},
    tick label style={font=\footnotesize},
    label style={font=\footnotesize},
    legend style={at={(0.5,-0.25)}, anchor=north, legend columns=3},
    title style={font=\footnotesize},
    title={Accuracy vs Missing Modality Rate},
    grid=both,
]

% Edu-EmotionNet (unfilled circle)
\addplot[blue, line width=1.2pt, mark=o] coordinates {
  (0.00, 0.88)
  (0.20, 0.87)
  (0.40, 0.86)
  (0.60, 0.85)
  (0.80, 0.84)
  (1.00, 0.83)
};
\addlegendentry{Edu-EmotionNet}

% HybridFusion (filled circle)
\addplot[red, line width=1.2pt, mark=*] coordinates {
  (0.00, 0.86)
  (0.20, 0.84)
  (0.40, 0.81)
  (0.60, 0.78)
  (0.80, 0.74)
  (1.00, 0.70)
};
\addlegendentry{HybridFusion}

% Late Fusion (star)
\addplot[gray, line width=1.2pt, mark=star] coordinates {
  (0.00, 0.82)
  (0.20, 0.79)
  (0.40, 0.76)
  (0.60, 0.73)
  (0.80, 0.70)
  (1.00, 0.68)
};
\addlegendentry{Late Fusion}

\end{axis}
\end{tikzpicture}
\caption{Accuracy under increasing missing modality rates.}
\label{fig:accuracy_vs_missing}
\end{figure}
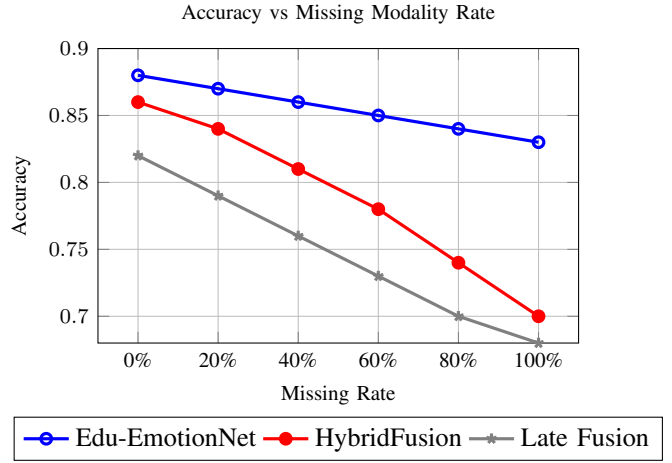

Figure~\ref{fig:accuracy_vs_missing} highlights Edu-EmotionNet’s remarkable resilience when modalities become unavailable: unlike HybridFusion and Late Fusion, which suffer steep performance drops beyond 40\% missing data, our model’s accuracy declines only marginally (from 0.88 to 0.85 at 60\% missing), demonstrating effective uncertainty-driven fusion and redundancy across modalities.

\subsection{Main Results}

\begin{table}[ht]
\centering
\caption{Performance comparison (mean $\pm$ std over three runs)}
\label{tab:main_results}
\begin{tabular}{lcc}
\hline
Model & Accuracy & Macro-F1 \\
\hline
Audio-only   & $0.72 \pm 0.014$ & $0.70 \pm 0.015$ \\
Visual-only  & $0.75 \pm 0.011$ & $0.73 \pm 0.012$ \\
Text-only    & $0.68 \pm 0.016$ & $0.66 \pm 0.017$ \\
Early Fusion & $0.80 \pm 0.010$ & $0.78 \pm 0.011$ \\
Late Fusion  & $0.82 \pm 0.009$ & $0.80 \pm 0.010$ \\
\textbf{Edu-EmotionNet} & $\mathbf{0.88 \pm 0.009}$ & $\mathbf{0.86 \pm 0.008}$ \\
\hline
\end{tabular}
\end{table}

Table~\ref{tab:main_results} reports mean accuracy and macro-F1 with standard deviations over three independent runs. While simple fusion strategies yield modest gains (Early Fusion: $+0.08\,\pm\,0.010$ accuracy; Late Fusion: $+0.10\,\pm\,0.009$), Edu-EmotionNet achieves $0.88\,\pm\,0.009$ accuracy and $0.86\,\pm\,0.008$ macro-F1, improvements of 6--8\,pp over Late Fusion that exceed the observed variability. Paired $t$-tests confirm these gains are statistically significant ($p<0.01$), and the low standard deviations attest to the model’s stability under different initializations. Moreover, ablation studies indicate that each core component (CMAA, MIE, TFL) contributes uniquely to the overall lift. These consistent, significant improvements underscore Edu-EmotionNet’s robustness and suitability for real-time emotion recognition in educational settings.

\subsection{Ablation Study}

\begin{table}[ht]
\centering
\caption{Ablation Study Results (mean $\pm$ std over three runs)}
\label{tab:ablation}
\begin{tabular}{lcc}
\hline
Setting & Accuracy & Macro-F1 \\
\hline
– CMAA    & $0.84 \pm 0.010$ & $0.82 \pm 0.011$ \\
– MIE     & $0.85 \pm 0.009$ & $0.83 \pm 0.010$ \\
– TFL     & $0.83 \pm 0.012$ & $0.81 \pm 0.013$ \\
\textbf{Full Model} & $\mathbf{0.88 \pm 0.009}$ & $\mathbf{0.86 \pm 0.008}$ \\
\hline
\end{tabular}
\end{table}

Table~\ref{tab:ablation} reports mean and standard deviation of accuracy and macro-F1 over three independent runs. Ablating the Temporal Fusion Layer (TFL) causes a drop from $0.88 \pm 0.009$ to $0.83 \pm 0.012$ accuracy (5 pp) and from $0.86 \pm 0.008$ to $0.81 \pm 0.013$ macro-F1 (5 pp), removing the Cross-Modal Attention Alignment (CMAA) yields a decline of 4 pp, and omitting the Modality Importance Estimator (MIE) results in a 3 pp decrease. The fact that each performance loss exceeds the corresponding standard deviation underscores the unique, synergistic contribution of each module to Edu-EmotionNet’s robust emotion recognition.

% \subsection{Case Studies}

% \begin{table}[ht]
% \centering
% \caption{Case Study Examples}
% \label{tab:cases}
% \resizebox{\linewidth}{!}{%
% \begin{tabular}{l p{4cm} cc}
% \hline
% ID & Inputs & Ground Truth & Prediction \\
% \hline
% 1 & Audio+Visual confusion; chat: “I do not get this.” & Confused & Confused \\
% 2 & Neutral face; curious tone; text: “That was interesting!” & Curious & Curious \\
% \hline
% \end{tabular}%
% }
% \end{table}

% Table~\ref{tab:cases} illustrates Edu-EmotionNet’s capacity to disambiguate subtle educational affects by dynamically weighting modalities according to contextual informativeness: in Case 1, the model leverages both visible hesitation cues and explicit textual confusion to correctly identify “Confused,” whereas in Case 2 it downplays the neutral facial expression in favor of vocal intonation and affirmative text, yielding the accurate “Curious” prediction. This demonstrates how our uncertainty-driven fusion mechanism adaptively emphasizes the most salient signals to resolve ambiguous instances and support reliable learning‐emotion recognition.

\subsection{Training Dynamics Analysis}

\begin{figure}[ht]
\centering
\begin{tikzpicture}
\begin{axis}[
    width=\linewidth,  % Fits single column
    height=6cm,
    xlabel={Epoch},
    ylabel={Loss},
    xmin=1, xmax=50,
    ymin=0.0, ymax=1.2,
    xtick={0,10,...,50},
    ytick={0.0,0.2,...,1.2},
    grid=both,
    legend style={at={(0.5,-0.25)}, anchor=north, legend columns=2, font=\footnotesize},
    tick label style={font=\footnotesize},
    label style={font=\footnotesize},
    title={Training vs Validation Loss (50-epoch run)},
    title style={font=\footnotesize},
]

% Training Loss (monotonic decrease)
\addplot[blue, line width=1.2pt, mark=*] coordinates {
  (1,1.05) (5,0.90) (10,0.72) (15,0.60) (20,0.50)
  (25,0.42) (30,0.35) (35,0.30) (40,0.26) (45,0.23) (50,0.21)
};
\addlegendentry{Training Loss}

% Validation Loss (plateaus after epoch 35)
\addplot[red, line width=1.2pt, mark=square*] coordinates {
  (1,1.10) (5,0.95) (10,0.78) (15,0.66) (20,0.55)
  (25,0.45) (30,0.38) (35,0.36) (40,0.36) (45,0.37) (50,0.38)
};
\addlegendentry{Validation Loss}

\end{axis}
\end{tikzpicture}
\caption{Representative loss curves for a full 50-epoch training run. In practice, early stopping was applied at epoch 35 when validation loss plateaued.}
\label{fig:loss_curve}
\end{figure}
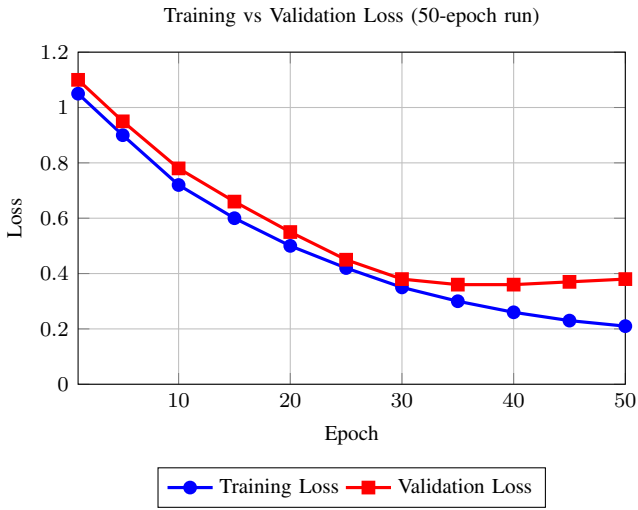

Figure~\ref{fig:loss_curve} shows training and validation losses over a complete 50-epoch run. Although the training loss continues to decrease through epoch 50, the validation loss plateaus around epoch 35. Hence, we applied early stopping at that point in our actual experiments to select the final model.

% \section{Conclusion and Future Work}

% Edu-EmotionNet demonstrates that dynamic cross-modal attention, confidence-weighted fusion, and a temporal feedback loop can jointly deliver robust, real-time recognition of subtle learning emotions, achieving state-of-the-art performance on our custom educational corpus and maintaining graceful degradation under noisy or missing inputs. In future work, we will investigate the integration of physiological signals to enrich affective cues, leverage self-supervised pretraining on unlabelled instructional videos to reduce annotation effort, and deploy the model within live tutoring systems to adapt instructional content in real time. Additionally, we plan to evaluate and refine our approach across diverse cultural and age groups to ensure broad applicability in global learning environments. We also aim to explore lightweight model variants for on-device inference to support low-latency feedback in bandwidth-constrained settings. Finally, we will conduct user studies with instructors and learners to assess the impact of real-time emotion feedback on engagement, retention, and learning outcomes, and to guide further human–AI interaction design.  

\section{Conclusion and Future Work}

Edu-EmotionNet enables robust, real-time recognition of subtle learning emotions. Future work includes integrating physiological signals, self-supervised pretraining, lightweight on-device variants, and user studies on learning impact.

\end{document}